\documentclass{article} 
\usepackage{iclr2027_conference,times}

\usepackage{amsmath,amsfonts,bm}

\def\eqref#1{equation~\ref{#1}}

\def\1{\bm{1}}

\DeclareMathAlphabet{\mathsfit}{\encodingdefault}{\sfdefault}{m}{sl}
\SetMathAlphabet{\mathsfit}{bold}{\encodingdefault}{\sfdefault}{bx}{n}

\newcommand{\E}{\mathbb{E}}

\newcommand{\KL}{D_{\mathrm{KL}}}

\usepackage{hyperref}
\usepackage{url}

\usepackage[pdftex]{graphicx}
\usepackage{amsmath}
\usepackage{amsfonts}
\usepackage{amssymb}
\usepackage{amsthm}
\usepackage{booktabs}
\usepackage{microtype}
\usepackage{xcolor}
\usepackage{algorithm}
\usepackage{algpseudocode}

\newtheorem*{theoremnn}{Theorem}
\newtheorem{theorem}{Theorem}

\newtheorem{corollary}[theorem]{Corollary}
\theoremstyle{definition}
\newtheorem{assumption}[theorem]{Assumption}

\theoremstyle{remark}

\providecommand{\E}{\mathbb{E}}

\providecommand{\N}{\mathcal{N}}
\providecommand{\X}{\mathcal{X}}
\providecommand{\Y}{\mathcal{Y}}
\providecommand{\KL}{\mathrm{KL}}

\title{Revisiting scaling laws for reward optimization}

\author{%
Ali Aouad \\
MIT \\
\texttt{maouad@mit.edu} \\
\And
Aymane El Gadarri \\
MIT \\
\texttt{aymelgad@mit.edu}
\And
Vivek F.~Farias \\
MIT \\
\texttt{vivekf@mit.edu} \\
}

\iclrfinalcopy 

\begin{document}

\maketitle
\lhead{Preprint}

\begin{abstract}
Scaling laws for optimization against reward models in AI alignment have pinned down how performance depends on optimization effort---\footnote{All dashes in this manuscript are human generated.}measured by a KL-divergence budget relative to a reference policy. Beyond a certain budget, over-optimization (or reward hacking) can arise: because we optimize against a proxy reward model (distinct from true rewards), performance can plateau or degrade. Naturally, the proxy reward's accuracy depends on how much preference data (often in the form of pairwise comparisons) was used to train it. However, existing research does not cleanly identify how performance jointly scales with the amount of training data and the  divergence budget. Our main contribution is to provide an empirically accurate and theoretically grounded scaling law  in such context. Performance roughly scales as $\Theta(\sqrt{\min\{\log(M),K\}})$, where $M$ is the number of comparisons in training data and $K$ is the policy's divergence budget. We develop an information-theoretic model to establish this upper bound  and prove it is tightly achievable through a constructive procedure. Informed by this, we conduct extensive empirical evaluations using a real-world annotation setup, whereby a large 70B gold reward model generates feedback data and proxy reward models are trained from less capable models (0.6B to 4B). Our scaling law provides an excellent fit (R2 from 97\% to 99\%), outperforms alternative specifications, and remains robust across model sizes, noise, and optimization procedures (best-of-$N$ or policy tilting). Our evidence suggests that reward optimization is analogous to a surprisingly simple selection task: choosing from  a sequence of IID Gaussian random variables using noisy preference feedback.
\end{abstract}

\section{Introduction}
\label{sec:intro}

Post-training typically entails learning, either explicitly or implicitly, a proxy `reward' model of user preferences~\citep{christiano2017deep,ouyang2022training,rafailov2023direct}. Such a model remains an imperfect proxy of true user preferences, learned from limited data. As such, it is natural that `over-optimizing' against such a proxy model may eventually lead to a decrease in the performance gained from post training~\citep{manheim2018categorizing,skalse2022defining}. A seminal paper by \citet{gao2023scaling} established empirical scaling laws for this phenomenon. They showed that performance first increased like the square root of the divergence of the post-trained policy from the pre-trained policy, and then beyond some specific divergence budget, performance stalled or decreased. With the advent of long running agents, one only expects such a problem to be exacerbated, potentially leading to catastrophic outcomes~\citep{vanroyconsequentialist}. This has been the topic of much recent research on superalignment. A consensus recipe that has emerged from the above research is that any post-trained policy should be constrained to within some pre-specified {\em divergence budget} of a reference policy to prevent over-optimization. While previous literature extensively studies tradeoffs between achievable performance and divergence budget, its dependence on the {\em amount of preference data} is not well-understood. This raises the natural question: 

\emph{How should the policy divergence budget be set as a function of the amount of data used to train the proxy reward model?}

It is natural, in particular, that a proxy model trained on limited preference data is more susceptible to over-optimization than one trained with a lot of such data. As such, we seek scaling laws that capture the amount data used to train the proxy reward model. In particular, we would like these scaling laws to inform the performance improvement this data affords us relative to a reference policy.  In addition, we wish to characterize how the aforementioned divergence budget ought to be scaled with the amount of preference data.  Answering these questions can provide important insights to clarify the fundamental limits of post-training and guide algorithmic scaling recipes in AI alignment.

\paragraph{Contributions.} Our work provides a theoretical and empirical characterization of how post-training performance jointly scales with the amount of preference data and the divergence budget.
\begin{itemize}
    \item {\em Simple theoretical derivation.} We develop an information-theoretic model that guides our subsequent empirical evaluations. The model yields a tight scaling law, roughly of the form $\Theta(\sqrt{\min\{\log(M),K\}})$, where $M$ is the number of preference comparisons in the training dataset (scale of feedback data) and $K$ is the KL-divergence budget (optimization effort). We give  constructive procedures that near-optimally attain this bound.
    \item {\em Accurate empirical scaling law.} We provide extensive empirical evidence that our post-training scaling law is highly accurate and provides a statistically better fit than alternative specifications, focusing on the novel predicted dependence $\Theta(\sqrt{\log M})$ on the amount of preference data. We use a real-world annotation setup whereby a large 70B model (serving as the gold reward model) generates preference data while proxy reward models are trained on these data using smaller models. The scaling law is robust across optimization procedures, noisy preferences, and model sizes.
    \item {\em Conceptual insight.} Our evidence in support of this scaling law connects post-training to a surprisingly simple selection problem: choosing among a sequence of IID Gaussian random variables from noisy preference data. 
    
\end{itemize}

\section{Related Literature}
\label{sec:related}

\paragraph{Alignment and overoptimization.} Overoptimization of a misspecified objective function (or reward hacking) is a long-standing concern in AI alignment~\citep{wiener1960some,leike2017ai}. Overoptimization can be defined as a regime in which increasing proxy reward strictly decreases true reward---an extreme form of Goodhart's Law whereby optimization against a certain metric erodes its efficacy~\citep{manheim2018categorizing}. To counter this, \cite{skalse2022defining} define a natural notion of {\em unhackable rewards} but find that it implies a very restrictive condition: any nontrivial unhackable rewards must induce the same policy ordering.  In MDPs, \citet{karwowski2024goodhart} give a geometric explanation of how the proxy and true-optimal policies diverge once optimization pushes the occupancy measure to a face of the polytope where the two rewards disagree. While this strand of literature focuses on mechanisms leading to overoptimization, we instead focus on quantifying its effect on achievable performance. Closer to us, \citet{vanroyconsequentialist} study a selection model in which an agent optimizing a consequentialist objective over a large outcome space incurs catastrophic performance with high probability. Our model shares the outcome selection set-up and information-theoretic approach, but we separate policy divergence and preference data budgets and construct policies that achieve matching lower bounds. Our goal is eventually to develop accurate scaling laws.

\paragraph{Post-training scaling laws.}\citet{gao2023scaling} provide the closest empirical antecedent to this work: in a synthetic setup where a fixed gold reward model stands in for human annotators, they measure gold reward as a function of KL-divergence from the base policy and find that it rises like $\sqrt{K}$ and then declines. However, while they vary the size of the preference dataset, they do not  identify an explicit scaling law in that dimension.\footnote{Quoting the authors, ``The scaling of $\alpha$ and $\beta$ with data size are not as cleanly described as for RM size scaling.''} An inverted-U shape also uncovered for direct policy optimization setups without a reward model \citep{rafailov2024scaling}. Pessimistic adjustments and robust ensembling methods can mitigate the perverse effects of overoptimization~\citep{coste2024reward, eisenstein2023helping,huang2025bon}. However, previous research often take the proxy model as exogenous and do not ask what specific divergence budget the preference data they are trained on can support. The dependence on $\sqrt{K}$ can be traced back to analyses of selection bias for adaptive data analyses~\citep{russo2016controlling}. 
Related work on AI alignment further characterizes tradeoffs between reward maximization and policy distribution shift~\citep{mroueh2024information, laidlaw2025correlated} under assumptions on the proxy rewards. Still, these works treat the underlying data as exogenous; in our language, the scaling is with respect to $K$ but it does not make explicit the dependence on the amount of preference data $M$. Instead our work reveals how information obtained from preference data constrains attainable rewards.

\paragraph{Learning from comparisons.}
The proxy in~\eqref{eqn:kl_constrained} is fit to pairwise preferences under a Bradley--Terry model \citep{bradley1952rank,christiano2017deep}. \citet{zhu2023principled} give the first sample complexity bounds for this pipeline, showing that the MLE converges for linear rewards under Bradley--Terry--Luce and Plackett--Luce, but that a policy trained on the plain MLE can fail while a pessimistic MLE succeeds under a coverage condition. This literature bounds how well $M$ comparisons pin down the reward; we instead ask how much reward gain $M$ comparisons can support at all, irrespective of the estimator. Closest to the combinatorial side of our question is work on selection from noisy comparisons: \citet{falahatgar2017maximum} establish the comparison complexity of identifying an approximate maximum, and \citet{chen2015spectral} that of top-$K$ recovery, with rates governed by the approximation tolerance and the gap at the top-$K$ boundary, respectively. However, these works do not explicitly quantify the corresponding expected reward gain nor they account for divergence constraints.
\section{Model and Results}
\label{sec:model}


Consider the following setup for alignment. For a prompt--response pair $(x,y) \in \X \times \Y$, human preferences are encoded by a \emph{gold} reward $r(x,y)$ that is never observed. In its place we have a proxy $\hat r(x,y)$, calibrated on a dataset of $M$ pairwise comparisons: for triples $(x,y,y')$, an annotator reports which of $y,y'$ is preferred. Given a base policy $\pi_0$ that assigns to each prompt a distribution over responses, an {\em aligned} policy $\hat\pi_K$ optimizes the proxy against a divergence budget,
\begin{equation}
\label{eqn:kl_constrained}
\max_{\pi} \; \E_{x,\, \pi}\!\left[\hat r(x,y)\right]
\quad \text{subject to} \quad
\E_{x}\!\left[\KL\!\left(\pi(\cdot\mid x) \,\middle\|\, \pi_{0}(\cdot \mid x)\right)\right] \;\leq\; K .
\end{equation}
$K$ constrains \emph{optimization pressure}. Because $\hat r$ is an imperfect proxy, driving expected proxy reward upward eventually drives expected gold reward downward. Writing $R(K) := \E_{x,\,\hat\pi_K}[r(x,y)]$ for the gold reward earned by the aligned policy, \citet{gao2023scaling} establish an empirical scaling law: $R(K)$ initially behaves like $\sqrt{K}$, and beyond some $K^\star$ is decreasing in $K$.

The proxy is itself calibrated from $M$ comparisons, so both $R(\cdot)$ and $K^\star$ must depend on $M$; the scaling law of \cite{gao2023scaling} does not address how. We seek to answer two related practical questions. First, {\em how does the scaling law depend on $M$ for a given $K$}, and in particular, {\em how should the divergence budget $K$ be set as a function of $M$?} And second, {\em what gold reward gain over the base policy does the resulting policy earn?}

\subsection{An Information-Theoretic Setup}
\label{sec:it_setup}


We propose here a simplified setup that will allow us to make theoretical progress on our questions; candidate scaling laws obtained from the analysis of this setup are then evaluated in the empirical setup of \citet{gao2023scaling} in \S\ref{sec:experiments}.
For a universe of $n$ potential responses, indexed by $i$, we associate with each an independent gold reward $R_i \sim \N(\mu,\sigma^2)$. We model a prompt as an \emph{admissible set} $A \subseteq [n]$, with $\Pr(i \in A) = p$ independently across $i$. Here $A$ is the set of responses that are candidate answers to the realized prompt, and $p$ is the fraction of the universe a prompt renders relevant.\footnote{If $A = \varnothing$ we allow a dummy fallback action $0$ with deterministic reward $R_0 = \mu$.} We remark that this setup is close to one studied recently by \citet{vanroyconsequentialist}.


A \emph{preference dataset} contains a label $Y_{ij}$ for each pair $(i,j)$ in a set $E$ of $M$ response pairs. The dataset is \emph{deterministic} when $Y_{ij}=\mathbf{1}\{R_i>R_j\}$, and \emph{random} when the labels are sampled conditionally on the rewards. We write $T:=\{(i,j,Y_{ij}):(i,j)\in E\}$ for the dataset.

\begin{assumption}[Information Structure] 
We assume that $E \perp (R,A)$ so that the pairs in the preference dataset are chosen without future knowledge of the admissible set or information about $R$. Moreover, $A \perp (R,T)$. 
\end{assumption}

We let $\pi_0(\cdot \mid A, E)$ be an arbitrary \emph{reward-blind base policy}. Specifically, an $A$-valued random variable $J$ distributed according to $\pi_0(\cdot \mid A, E)$ is dependent on $A$ and potentially, $E$ but independent of $R$. Consequently, $\E_{\pi_0}[R_J] = \mu$. An \emph{informed policy} $\pi(\cdot \mid T, A)$ selects an index $J \in A$ after seeing the preference dataset $T$ and the admissible set $A$. Define the expected policy divergence and the reward gain by
\begin{equation}
\label{eqn:K_and_Delta}
K(\pi;\pi_0) \;\triangleq\; \E\!\left[\KL\!\left(\pi(\cdot\mid T,A) \,\middle\|\, \pi_0(\cdot \mid A, E)\right)\right],
\qquad
\Delta(\pi) \;\triangleq\; \E_\pi[R_J] - \mu .
\end{equation}

We next present our results; taken together they provide a theoretical answer to the practical questions posed above.

%
%

%

\subsection{A Preview of Our Results}
\label{sec:upper_bounds}
We present here an informal overview of our key result, which holds for both deterministic and random preference datasets:

\begin{theoremnn}[Informal]
Suppose $K(\pi;\pi_0) \leq K$ and $|E| \leq M$. Then, for any such $\pi$, 
\[
\Delta(\pi) 
\;\lesssim\; \sigma\sqrt{2\min\{K, \log(pM)\}} .
\]
Moreover, there exists a preference dataset $E$ and a resulting proxy $\hat r$ for which the optimized policy $\pi$ satisfies 
\[
\Delta(\pi) 
\;\gtrsim\; 
(1-o(1)) 
\,\sigma\sqrt{2\min\{K, \log(p\sqrt{M})\}} .
\]
\end{theoremnn}

The asymptotic regime implicit in the statements above requires that $\min\{K, \log(p\sqrt{M})\}$ grow large.\footnote{{ The dependence on $p$ is not central to our analysis; $p$ is a stylized measure of how much overlap there is in the admissible responses of distinct prompts---a notion of how well preference data generalizes across prompts.}} That is, the regime is one where both the divergence budget, and the effective number of comparisons for a prompt grows. We briefly discuss the implications of the above result here. First, observe that the lower and upper bounds on $\Delta(\pi)$ are tight: they are within a constant factor of $\sqrt{2}$ of each other. As such, the bounds explain precisely the $\sqrt{K}$ scaling law established by \citet{gao2023scaling}. The bounds also prescribe how to pick the divergence budget $K$: specifically, we want $K \sim \log(pM)$: any slower and we leave improvement opportunities on the table; any faster and we risk over-optimizing against the proxy. The next section formalizes the above theorem and presents its proof, following which we turn to empirically verifying that the above result holds in a (larger scale) variant of the scaling law study performed by \citet{gao2023scaling}. 

\section{Scaling laws for reward optimization}
\subsection{Reward optimization scaling law under deterministic preferences}

We now make the preceding bounds precise. We allow the informed policy to use the comparison data directly, so our upper bound applies regardless of how a proxy is fit. For the lower bound, we construct a proxy and a base policy that attain the stated gain. Write $\hat r(i;T)$ for the proxy reward assigned to response $i$ using the comparison data $T$. As in~\eqref{eqn:kl_constrained}, the aligned policy $\hat\pi_K$ maximizes expected proxy reward within divergence budget $K$.

We assume for simplicity that $n \geq \sqrt{2M}$ throughout the analysis. This ensures enough possible responses for our construction and leaves the scaling law in $M$ and $K$ unchanged. In practice, it is plausible to assume that the number of possible responses is much larger than our preference dataset.

\begin{theorem}[Deterministic preferences]
\label{thm:offline_alignment}
Under the information structure above, every reward-blind base policy $\pi_0$, deterministic preference dataset generated on pairs $E$ with $|E|\leq M$, and informed policy with $K(\pi;\pi_0)\leq K$ must satisfy
\begin{equation}
\label{eqn:offline_upper_bound}
\Delta(\pi)\leq\sigma\sqrt{2\min\{K,\log(1+2pM)\}}.
\end{equation}
Reciprocally, there is a universal constant $c>0$ such that, whenever $p\sqrt{2M}\geq4$, for every $K\geq0$ there exist a set of comparison pairs $E$ with $|E|\leq M$, a reward-blind base policy, and a proxy such that, for the resulting deterministic preference dataset, every optimizer $\hat\pi_K$ satisfies
\begin{equation}
\label{eqn:offline_lower_bound}
\Delta(\hat\pi_K)\geq c\,\sigma\sqrt{\min\{K,\log(p\sqrt{2M})\}}.
\end{equation}
\end{theorem}

The common upper bound is proved below, the lower-bound construction is given in Appendix~\ref{app:exact_proof}.

\subsection{The same scaling law holds under noisy preferences}
So far, we assumed that preferences were deterministic: a response is recorded as preferred whenever it has the larger gold reward. In reality, preference data is  noisy. For each $(i,j)\in E$, let $Y_{ij}=1$ indicate that $i$ is preferred to $j$. As is standard in RLHF, we model preference labels with the Bradley--Terry model~\citep{bradley1952rank}:
\begin{equation}
\label{eqn:noisy_comparison_model}
\Pr(Y_{ij}=1\mid R,E)=\frac{1}{1+\exp(-(R_i-R_j)/\beta)},
\qquad \beta>0.
\end{equation}
Here $\beta>0$ controls the preference noise: larger $\beta$ makes labels less sensitive to reward differences. The finite-budget lower bound now depends on $\beta/\sigma$, the noise scale relative to the variation in gold rewards.

\begin{theorem}[Noisy preferences]
\label{thm:noisy_alignment}
Suppose preferences follow the Bradley-Terry model with conditionally independent labels and noise scale $\beta>0$. Under the information structure above, every reward-blind base policy $\pi_0$, random preference dataset generated on pairs $E$ with $|E|\leq M$, and informed policy with $K(\pi;\pi_0)\leq K$ satisfy
\begin{equation}
\label{eqn:noisy_upper_bound}
\Delta(\pi)\leq\sigma\sqrt{2\min\{K,\log(1+2pM)\}}.
\end{equation}
Reciprocally, there is a constant $c(\beta/\sigma)>0$, depending only on $\beta/\sigma$, such that, whenever $p\sqrt{2M}\geq4$, for every $K\geq0$ there exist a set of comparison pairs $E$ with $|E|\leq M$, a reward-blind base policy, and a proxy such that, for the resulting random preference dataset, every optimizer $\hat\pi_K$ satisfies
\begin{equation}
\label{eqn:noisy_lower_bound}
\Delta(\hat\pi_K)\geq c(\beta/\sigma)\,\sigma\sqrt{\min\{K,\log(p\sqrt{2M})\}}.
\end{equation}
\end{theorem}

Both lower bounds fix a set of $m=\lfloor\sqrt{2M}\rfloor$ responses and compare each pair once, using $\binom{m}{2}\leq M$ comparisons. Deterministic preferences reveal the order of these responses by gold reward. With noisy preferences, we score each response by the fraction of its $m-1$ comparisons in which it is preferred. Within any nonempty subset of these responses, we then select the response with the highest score. For fixed $\sigma,\beta>0$, the expected gap between its gold reward and the largest gold reward in that subset tends to zero as $m$ grows. Appendix~\ref{app:noisy_proof} gives the full finite-budget argument. We prove the common upper bound below.
\begin{proof}
    Consider a preference dataset generated on a possibly random pair set $E$, with $|E|\leq M$, and an informed policy $\pi$ satisfying $K(\pi;\pi_0)\leq K$. Both arguments below apply to deterministic preferences and to the Bradley--Terry model.

\paragraph{The divergence budget.}
Let $P_\pi$ and $P_0$ denote the joint laws of $(R,E,T,A,J)$ under $\pi$ and $\pi_0$. These laws differ only in the conditional distribution of the chosen $J$. The chain rule for relative entropy therefore gives
\[
\KL(P_\pi\|P_0)
=\E\!\left[\KL\!\left(\pi(\cdot\mid T,A)\,\middle\|\,\pi_0(\cdot\mid A,E)\right)\right]
\leq K.
\]
Under $P_0$, the selected index is independent of the rewards. Thus, for every $t>0$,
\[
\E_{P_0}\!\left[e^{t(R_J-\mu)}\right]
=\Pr(A=\varnothing)+\Pr(A\ne\varnothing)e^{t^2\sigma^2/2}
\leq e^{t^2\sigma^2/2}.
\]
The entropy inequality $\E_P f\leq\KL(P\|Q)+\log\E_Q e^f$, applied with $f=t(R_J-\mu)$, gives
\[
\Delta(\pi)\leq\frac Kt+\frac{t\sigma^2}{2}.
\]
Taking the infimum over $t>0$ gives $\Delta(\pi)\leq\sigma\sqrt{2K}$, including $K=0$ by letting $t\downarrow0$.

\paragraph{The size of the preference dataset.}
Let $V(E)$ be the set of responses appearing in $E$, so $|V(E)|\leq2M$. Conditional on $E$, the labels depend only on rewards in $V(E)$ and, in the noisy model, independent preference randomness. Every response outside $V(E)$ therefore has conditional mean $\mu$ given $(T,A)$. Writing $\pi_i=\pi(i\mid T,A)$ and using the independence of the policy's randomization from $R$ given $(T,A)$, we obtain
\[
\begin{split}
\Delta(\pi)
&=\E\!\left[\sum_{i\in A\cap V(E)}\pi_i(R_i-\mu)\right]\\
&\leq\E[Z],\qquad
Z:=\max\bigl(\{0\}\cup\{R_i-\mu:i\in A\cap V(E)\}\bigr).
\end{split}
\]
For $t>0$, independence of $E$, $A$, and the rewards gives
\[
\begin{split}
\E \left[e^{tZ}\right]
&\leq1+\E\!\left[\sum_{i\in A\cap V(E)}e^{t(R_i-\mu)}\right]\\
&=1+p\E|V(E)|e^{t^2\sigma^2/2}
\leq(1+2pM)e^{t^2\sigma^2/2}.
\end{split}
\]
Jensen's inequality and minimization over $t$ now imply
\[
\Delta(\pi)\leq\inf_{t>0}\left\{\frac{\log(1+2pM)}t+\frac{t\sigma^2}{2}\right\}
=\sigma\sqrt{2\log(1+2pM)}.
\]
Combining the two bounds proves the upper bounds in Theorems~\ref{thm:offline_alignment} and~\ref{thm:noisy_alignment}. Notice that neither argument requires the size assumptions used for the lower bounds.
\end{proof}

The constants in the two finite-budget guarantees differ, but both admit the same sharp asymptotic bound.

\begin{corollary}[Asymptotic bound on performance]
\label{cor:comparison_asymptotics}
Fix $\sigma>0$ and, for Bradley--Terry preferences, $\beta>0$. Under either preference model, along any sequence $(n,M,p,K)$ with $K\to\infty$ and $p\sqrt{2M}\to\infty$, there exist a pair set $E$ with $|E|\leq M$, a reward-blind base policy, and a proxy such that, for the resulting preference dataset, every optimizer $\hat\pi_K$ satisfies
\begin{equation}
\label{eqn:comparison_asymptotic_lower}
\Delta(\hat\pi_K)\geq(1-o(1))\,\sigma\sqrt{2\min\{K,\log(p\sqrt{2M})\}}.
\end{equation}
\end{corollary}

\section{Experiments}
\label{sec:experiments}

\subsection{Summary}
In this section, we evaluate our proposed scaling law empirically and test its accuracy against other  specifications. We use a real-world annotation set-up and generate preference data from a large gold reward model. We train proxy reward models with 0.6B, 1.7B, and 4B parameters and compare alternative scaling functions. The square-root-logarithmic form fits well across model sizes, preference noise, and optimization methods. It significantly outperforms linear and square-root growth in all tested settings

\subsection{Setup}\label{sec:experimental_setup}
\paragraph{Models and data.}
Following \citet{gao2023scaling}, we study reward optimization using preferences supplied by a fixed gold reward model, Llama-3.3-Nemotron-70B-Reward. Our proxies are Qwen3 base models with 0.6B, 1.7B, and 4B parameters, each with a scalar reward head. Our training corpus contains 1,000 first-turn prompts from the Arena 140K dataset \cite{arena140k}, a highly diverse dataset spanning code, mathematics, and other topics (the full mixture is described in Appendix \ref{app:data}). We sample 600 responses per prompt  from the fixed reference policy Qwen3-4B-Instruct at temperature~1.

\paragraph{Preference data.}
We construct 300 disjoint response pairs per prompt, summing to 300,000 offline comparisons in our largest dataset. We exclude identical-response pairs and duplicate text pairs. In the noiseless setting, the response with highest gold reward is preferred, with exact ties broken uniformly, as in \cite{gao2023scaling}. For the noisy-label experiments, we sample a Bradley-Terry preference with probability $\Pr(y_a\succ y_b\mid x)=\sigma(r(x,y_a)-r(x,y_b))$.

\paragraph{Training and comparison budgets.}
The main scaling experiment increases $M$ from 600 to 300,000 by expanding the set of prompts in a fixed, reward-independent order, while holding the number of comparisons per prompt at 300. We fine-tune all parameters for one pass using the pairwise logistic loss and AdamW at constant learning rate $10^{-5}$. Each policy update contains 600 comparisons from two complete prompt groups. 

\paragraph{Evaluation.}
We evaluate on 128 prompts disjoint from training, with 8,192 responses each. Best-of-$N$ selects the highest-proxy-reward response among $N$ candidates. We sweep logarithmically spaced $N$ up to 8,192, corresponding to the conventional optimization-pressure coordinate $K_N=\log N-(N-1)/N$ \citep{gao2023scaling}. We estimate gold and proxy rewards using the unbiased subset-average estimator of \citet[Appendix~I]{nakano_webgpt_2021}, which highly reduces the variance of the estimates. Reward gains are measured relative to the mean reward at $N=1$ over the full response pool for each prompt, and are normalized by a fixed reference gold standard deviation.

\subsection{The peak reward follows the predicted square root logarithmic scaling}

The comparison-limited regime of Section~\ref{sec:upper_bounds} predicts square-root-logarithmic growth of achievable gold-reward gains. We first test this scaling by varying the size of preference dataset $M$ sizes from $600$ to $300,000$. For each trained proxy reward ($M$), we determine the best-of-N response model, varying $N\in \mathcal{G}$ in our evaluated grid bounded by $N\leq 8{,}192$. We tune $N$ as the observed peak of $\widehat\Delta_{\max}(M)=\max_{N\in\mathcal{G}}\widehat\Delta_M(N)$, where $\widehat\Delta_M(N)$ is the mean standardized gold-reward gain.

\begin{table}[t]
    \centering
    \small
    \setlength{\tabcolsep}{3.2pt}
    \caption{\textbf{Scaling fits across model size, preference noise, and optimization method.} Each $R^2$ fits $a+b f(M)$ on all evaluation prompts, bold marks the largest $R^2$ in each row. The final three columns report Holm-adjusted, two-sided paired $t$-test $p$-values against $\sqrt{\log M}$.}
    \label{tab:scaling_summary}
    \begin{tabular}{lllrrrrrrr}
        \toprule
        & & & \multicolumn{4}{c}{$R^2$} & \multicolumn{3}{c}{$p$-value vs.\ $\sqrt{\log M}$} \\
        \cmidrule(lr){4-7}\cmidrule(lr){8-10}
        Proxy & Noise & Method & $\sqrt{\log M}$ & $\log M$ & $M$ & $M^{1/2}$ & $\log M$ & $M$ & $M^{1/2}$ \\
        \midrule
        0.6B & None & BoN & \textbf{0.992} & 0.990 & 0.571 & 0.803 & 1.000 & $6.1\times 10^{-20}$ & $3.4\times 10^{-9}$ \\
        0.6B & None & Tilt & \textbf{0.991} & 0.988 & 0.564 & 0.795 & 1.000 & $2.7\times 10^{-20}$ & $3.1\times 10^{-10}$ \\
        0.6B & BTL & BoN & 0.951 & \textbf{0.954} & 0.537 & 0.775 & 0.392 & $1.8\times 10^{-20}$ & $7.3\times 10^{-11}$ \\
        0.6B & BTL & Tilt & 0.947 & \textbf{0.949} & 0.527 & 0.767 & 0.708 & $1.3\times 10^{-20}$ & $4.7\times 10^{-11}$ \\
        \midrule
        1.7B & None & BoN & \textbf{0.989} & 0.981 & 0.520 & 0.757 & 0.023 & $8.7\times 10^{-25}$ & $2.2\times 10^{-14}$ \\
        1.7B & None & Tilt & \textbf{0.989} & 0.982 & 0.519 & 0.758 & 0.025 & $2.3\times 10^{-24}$ & $4.3\times 10^{-15}$ \\
        1.7B & BTL & BoN & \textbf{0.963} & 0.957 & 0.511 & 0.743 & 1.000 & $4.4\times 10^{-23}$ & $2.0\times 10^{-10}$ \\
        1.7B & BTL & Tilt & \textbf{0.953} & 0.948 & 0.502 & 0.734 & 1.000 & $1.4\times 10^{-22}$ & $3.4\times 10^{-10}$ \\
        \midrule
        4B & None & BoN & \textbf{0.972} & 0.964 & 0.476 & 0.721 & 0.003 & $2.4\times 10^{-28}$ & $2.7\times 10^{-17}$ \\
        4B & None & Tilt & \textbf{0.970} & 0.963 & 0.476 & 0.721 & 0.002 & $4.2\times 10^{-30}$ & $1.5\times 10^{-18}$ \\
        \bottomrule
    \end{tabular}
\end{table}

Figure~\ref{fig:gold_reward_scaling_0.6B} shows that the peak gold reward closely follows the predicted scaling law. Figure~\ref{fig:gold_reward_alternative_fits_0.6B} compares the predicted form with affine fits $a+b f(M)$ using the same models and two free parameters for every function. The square-root-logarithmic form gives the best fit, with $R^2=0.992$, compared with 0.990, 0.803, and 0.571 for $\log M$, $\sqrt{M}$, and $M$, respectively.

\begin{figure}[t]
    \centering
    \includegraphics[width=\linewidth]{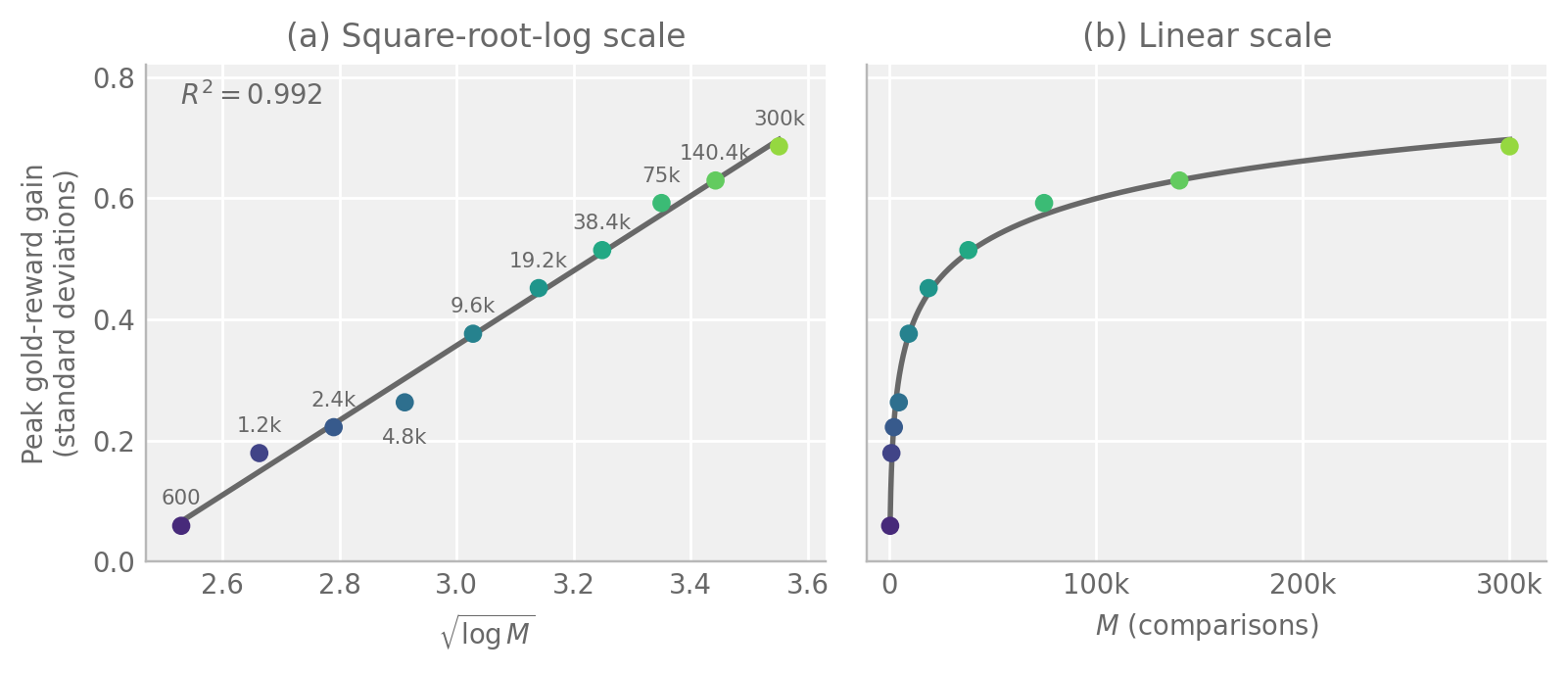}
    \caption{\textbf{Gold reward follows square-root-logarithmic comparison-budget scaling.} Ten 0.6B models are selected by uniform spacing in $\sqrt{\log M}$. Each point is the maximum mean gold-reward gain over the evaluated best-of-$N$ grid. Both panels show the same fitted function $a+b\sqrt{\log M}$: left on a square-root-logarithmic axis, right on a linear $M$ axis. Point labels denote comparison counts.}
    \label{fig:gold_reward_scaling_0.6B}
\end{figure}

\begin{figure}[t]
    \centering
    \includegraphics[width=0.85\linewidth]{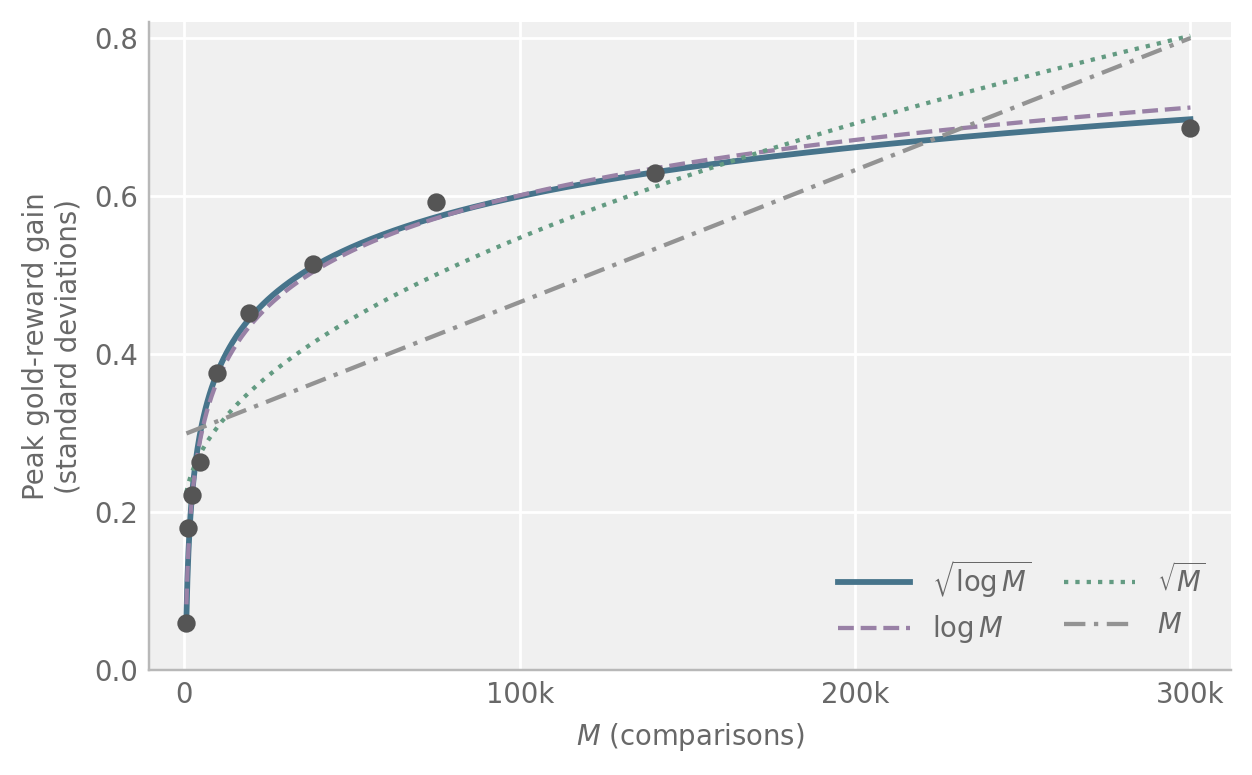}
    \caption{\textbf{Comparison of alternative scaling functions.} All curves are fitted to the same ten models as Figure~\ref{fig:gold_reward_scaling_0.6B}. The fits $a+b f(M)$ are shown on a common linear comparison-count axis.}
    \label{fig:gold_reward_alternative_fits_0.6B}
\end{figure}

Table~\ref{tab:scaling_summary} summarizes the fits across proxy sizes, preference noise, and optimization methods, which we examine below.
\subsection{Robustness to the optimization algorithm}
\label{sec:tilted_policy_scaling}

We test robustness to the optimization algorithm by replacing best-of-$N$ selection with exponential tilting, an idealized form of KL-regularized reinforcement learning that approximates its optimal policy using samples from the reference policy. For inverse temperature $\beta>0$, the objective $\max_\pi\{\E_\pi[\hat r_M]-\beta^{-1}\KL(\pi\|\pi_0)\}$ has the exponentially tilted optimizer $\pi_{\beta,M}(y\mid x)\propto\pi_0(y\mid x)\exp(\beta\hat r_M(x,y))$ \citep{rafailov2023direct}. We approximate this policy using responses from the reference policy. The same exponential reweighting is used in soft best-of-$N$ sampling to approximate KL-regularized alignment \citep{verdun_soft_2025}, including empirical studies with imperfect reward models \citep{aminian_best--n_2025}.

For each evaluation prompt $x$, we place uniform reference mass on its $S=8{,}192$ sampled responses and assign the tilted policy weights
\begin{equation}
\label{eq:empirical_tilt}
w_{\beta,M,i}(x)
=\frac{\exp\!\left(\beta\hat r_M(x,y_i)\right)}
{\sum_{j=1}^{S}\exp\!\left(\beta\hat r_M(x,y_j)\right)}.
\end{equation}
We compute the weighted gold reward and conditional KL exactly on each pool, averaging over the same test prompts and retaining the gold-reward normalization from Section~\ref{sec:experimental_setup}.

The right column of Figure~\ref{fig:all_scaling_fits} in Appendix~\ref{app:additional_exp} shows the same square-root-logarithmic scaling under exponential tilting.

\subsection{Robustness to preference noise}
\label{sec:noisy_preference_scaling}

We next test robustness to preference noise using the Bradley--Terry--Luce (BTL) model \citep{bradley1952rank}. Labels are drawn once and held fixed throughout training, 33.6\% of the 300,000 labels differ from their noiseless counterparts. For each model size, we keep the same model initialization, comparison stream, and evaluation pools, so the comparison isolates the effect of preference noise.
Panels (c), (d), (g), and (h) of Figure~\ref{fig:all_scaling_fits} show that the square-root-logarithmic trend persists under noisy preferences.

The fit is slightly less tight than in the noiseless setting, its agreement with the scaling predicted by our theory is on par with the theory's expected behavior for noisy labels. In particular, the noisy-label theorem predicts the same square-root-logarithmic scaling, but only in a higher-comparison regime. We hypothesize that in practice, additional samples are needed for the noise to average out and for the asymptotic scaling to emerge more cleanly.

\subsection{Robustness to proxy model size}
\label{sec:model_size_scaling}

Finally, we increase the proxy size from 0.6B to 1.7B and 4B parameters within the Qwen3 base-model family. Both proxies use the same noiseless preferences, comparison order, and evaluation pools.

Table~\ref{tab:scaling_summary} shows that the square-root-logarithmic form remains the best fit among the four specifications for both larger proxies and both optimization methods. For 1.7B, $R^2=0.989$ under both methods; for 4B, $R^2=0.972$ under best-of-$N$ and $0.970$ under tilting. The paired prediction tests also favor this form over $\log M$, $\sqrt{M}$, and $M$ for both larger proxies under both methods, with Holm-adjusted $p<0.05$ in each comparison.



\section{Conclusion}
\label{sec:conclusion}
Our work provides a theoretical and empirical characterization of how post-training performance scales jointly with preference data and optimization effort. A simple information-theoretic model yields the scaling $\sqrt{\min(\log (M),K})$ , with constructive procedures that attain this order. Our experiments show that the predicted dependence on preference data provides a strong fit across model sizes, noisy preferences, and optimization procedures. Extending the theory to more general reward distributions and adaptive preference collection, and testing the scaling with human preferences and policy training, are natural directions for future work.

\bibliographystyle{iclr2027_conference}
\bibliography{references}

@inproceedings{gao2023scaling,
  title={Scaling Laws for Reward Model Overoptimization},
  author={Gao, Leo and Schulman, John and Hilton, Jacob},
  booktitle={Proceedings of the 40th International Conference on Machine Learning},
  series={Proceedings of Machine Learning Research},
  volume={202},
  pages={10835--10866},
  year={2023},
  publisher={PMLR}
}

@misc{arena140k,
  title        = {Arena Human Preference 140K},
  author       = {{LMArena}},
  year         = {2025},
  howpublished = {\url{https://huggingface.co/datasets/lmarena-ai/arena-human-preference-140k}},
  note         = {Hugging Face dataset}
}

@article{rafailov2023direct,
  title={Direct preference optimization: Your language model is secretly a reward model},
  author={Rafailov, Rafael and Sharma, Archit and Mitchell, Eric and Manning, Christopher D and Ermon, Stefano and Finn, Chelsea},
  journal={Advances in neural information processing systems},
  volume={36},
  pages={53728--53741},
  year={2023}
}

@article{ouyang2022training,
  title={Training language models to follow instructions with human feedback},
  author={Ouyang, Long and Wu, Jeffrey and Jiang, Xu and Almeida, Diogo and Wainwright, Carroll and Mishkin, Pamela and Zhang, Chong and Agarwal, Sandhini and Slama, Katarina and Ray, Alex and others},
  journal={Advances in neural information processing systems},
  volume={35},
  pages={27730--27744},
  year={2022}
}

@article{mroueh2024information,
  title={Information theoretic guarantees for policy alignment in large language models},
  author={Mroueh, Youssef},
  journal={arXiv preprint arXiv:2406.05883},
  year={2024}
}

@inproceedings{laidlaw2025correlated,
  title={Correlated proxies: A new definition and improved mitigation for reward hacking},
  author={Laidlaw, Cassidy and Singhal, Shivam and Dragan, Anca},
  booktitle={International Conference on Learning Representations},
  volume={2025},
  year={2025}
}

@inproceedings{russo2016controlling,
  title={Controlling bias in adaptive data analysis using information theory},
  author={Russo, Daniel and Zou, James},
  booktitle={Artificial Intelligence and Statistics},
  pages={1232--1240},
  year={2016},
  organization={PMLR}
}

@article{manheim2018categorizing,
  title={Categorizing variants of Goodhart's Law},
  author={Manheim, David and Garrabrant, Scott},
  journal={arXiv preprint arXiv:1803.04585},
  year={2018}
}

@article{wiener1960some,
  title={Some moral and technical consequences of automation: As machines learn they may develop unforeseen strategies at rates that baffle their programmers.},
  author={Wiener, Norbert},
  journal={Science},
  volume={131},
  number={3410},
  pages={1355--1358},
  year={1960},
  publisher={American Association for the Advancement of Science}
}

@article{leike2017ai,
  title={AI safety gridworlds},
  author={Leike, Jan and Martic, Miljan and Krakovna, Victoria and Ortega, Pedro A and Everitt, Tom and Lefrancq, Andrew and Orseau, Laurent and Legg, Shane},
  journal={arXiv preprint arXiv:1711.09883},
  year={2017}
}

@article{bradley1952rank,
  title={Rank analysis of incomplete block designs: I. The method of paired comparisons},
  author={Bradley, Ralph Allan and Terry, Milton E.},
  journal={Biometrika},
  volume={39},
  number={3/4},
  pages={324--345},
  year={1952}
}

@inproceedings{christiano2017deep,
  title={Deep reinforcement learning from human preferences},
  author={Christiano, Paul F. and Leike, Jan and Brown, Tom and Martic, Miljan and Legg, Shane and Amodei, Dario},
  booktitle={Advances in Neural Information Processing Systems},
  volume={30},
  year={2017}
}

@article{vanroyconsequentialist,
  title={Consequentialist Objectives and Catastrophe},
  author={Marklund, Henrik and Infanger, Alex and Van Roy, Benjamin},
  journal={arXiv preprint arXiv:2603.15017},
  year={2026}
}

@inproceedings{rafailov2024scaling,
  title={Scaling Laws for Reward Model Overoptimization in Direct Alignment Algorithms},
  author={Rafailov, Rafael and Chittepu, Yaswanth and Park, Ryan and Sikchi, Harshit and Hejna, Joey and Knox, Bradley and Finn, Chelsea and Niekum, Scott},
  booktitle={Advances in Neural Information Processing Systems},
  volume={37},
  year={2024}
}

@inproceedings{coste2024reward,
  title={Reward Model Ensembles Help Mitigate Overoptimization},
  author={Coste, Thomas and Anwar, Usman and Kirk, Robert and Krueger, David},
  booktitle={International Conference on Learning Representations},
  year={2024}
}

@article{eisenstein2023helping,
  title={Helping or Herding? Reward Model Ensembles Mitigate but do not Eliminate Reward Hacking},
  author={Eisenstein, Jacob and Nagpal, Chirag and Agarwal, Alekh and Beirami, Ahmad and D'Amour, Alex and Dvijotham, DJ and Fisch, Adam and Heller, Katherine and Pfohl, Stephen and Ramachandran, Deepak and Shaw, Peter and Berant, Jonathan},
  journal={arXiv preprint arXiv:2312.09244},
  year={2023}
}

@inproceedings{skalse2022defining,
  title={Defining and Characterizing Reward Hacking},
  author={Skalse, Joar and Howe, Nikolaus H. R. and Krasheninnikov, Dmitrii and Krueger, David},
  booktitle={Advances in Neural Information Processing Systems},
  volume={35},
  year={2022}
}

@inproceedings{karwowski2024goodhart,
  title={Goodhart's Law in Reinforcement Learning},
  author={Karwowski, Jacek and Hayman, Oliver and Bai, Xingjian and Kiendlhofer, Klaus and Griffin, Charlie and Skalse, Joar},
  booktitle={International Conference on Learning Representations},
  year={2024}
}

@inproceedings{huang2025bon,
  title={Is Best-of-N the Best of Them? Coverage, Scaling, and Optimality in Inference-Time Alignment},
  author={Huang, Audrey and Block, Adam and Liu, Qinghua and Jiang, Nan and Krishnamurthy, Akshay and Foster, Dylan J.},
  booktitle={Proceedings of the 42nd International Conference on Machine Learning},
  series={Proceedings of Machine Learning Research},
  volume={267},
  year={2025},
  publisher={PMLR}
}

@inproceedings{zhu2023principled,
  title={Principled Reinforcement Learning with Human Feedback from Pairwise or $K$-wise Comparisons},
  author={Zhu, Banghua and Jordan, Michael and Jiao, Jiantao},
  booktitle={Proceedings of the 40th International Conference on Machine Learning},
  series={Proceedings of Machine Learning Research},
  volume={202},
  year={2023},
  publisher={PMLR}
}

@inproceedings{falahatgar2017maximum,
  title={Maximum Selection and Ranking under Noisy Comparisons},
  author={Falahatgar, Moein and Orlitsky, Alon and Pichapati, Venkatadheeraj and Suresh, Ananda Theertha},
  booktitle={Proceedings of the 34th International Conference on Machine Learning},
  series={Proceedings of Machine Learning Research},
  volume={70},
  year={2017},
  publisher={PMLR}
}

@inproceedings{chen2015spectral,
  title={Spectral {MLE}: Top-$K$ Rank Aggregation from Pairwise Comparisons},
  author={Chen, Yuxin and Suh, Changho},
  booktitle={Proceedings of the 32nd International Conference on Machine Learning},
  series={Proceedings of Machine Learning Research},
  volume={37},
  year={2015},
  publisher={PMLR}
}

@misc{nakano_webgpt_2021,
    title = {{WebGPT}: {Browser}-assisted question-answering with human feedback},
    shorttitle = {{WebGPT}},
    url = {https://arxiv.org/abs/2112.09332v3},
    language = {en},
    urldate = {2026-09-24},
    journal = {arXiv.org},
    author = {Nakano, Reiichiro and Hilton, Jacob and Balaji, Suchir and Wu, Jeff and Ouyang, Long and Kim, Christina and Hesse, Christopher and Jain, Shantanu and Kosaraju, Vineet and Saunders, William and Jiang, Xu and Cobbe, Karl and Eloundou, Tyna and Krueger, Gretchen and Button, Kevin and Knight, Matthew and Chess, Benjamin and Schulman, John},
    month = dec,
    year = {2021},
}

@misc{verdun_soft_2025,
    title = {Soft {Best}-of-n {Sampling} for {Model} {Alignment}},
    url = {http://arxiv.org/abs/2505.03156},
    doi = {10.48550/arXiv.2505.03156},
    urldate = {2026-09-01},
    publisher = {arXiv},
    author = {Verdun, Claudio Mayrink and Oesterling, Alex and Lakkaraju, Himabindu and Calmon, Flavio P.},
    month = may,
    year = {2025},
    note = {arXiv:2505.03156 [cs.IT]},
}

@misc{aminian_best--n_2025,
    title = {Best-of-{N} through the {Smoothing} {Lens}: {KL} {Divergence} and {Regret} {Analysis}},
    shorttitle = {Best-of-{N} through the {Smoothing} {Lens}},
    url = {http://arxiv.org/abs/2507.05913},
    doi = {10.48550/arXiv.2507.05913},
    urldate = {2026-09-24},
    publisher = {arXiv},
    author = {Aminian, Gholamali and Shenfeld, Idan and Asadi, Amir R. and Beirami, Ahmad and Mroueh, Youssef},
    month = jul,
    year = {2025},
    note = {arXiv:2507.05913 [stat.ML]},
}

\appendix
\section{Proofs}
\label{app:proofs}

We first establish the upper bound shared by the two preference models. We then prove the lower bound for exact preferences, extend the construction to noisy preferences, and derive the common asymptotic bound. Throughout the lower-bound proofs, $\sigma>0$. If $\sigma=0$, every policy has zero gain.

\subsection{The reward optimization upper bound}
\label{app:common_upper}

Consider a preference dataset generated on a possibly random pair set $E$, with $|E|\leq M$, and an informed policy $\pi$ satisfying $K(\pi;\pi_0)\leq K$. Both arguments below apply to exact preferences and to the Bradley--Terry model.

\paragraph{The divergence budget.}
Let $P_\pi$ and $P_0$ denote the joint laws of $(R,E,T,A,J)$ under $\pi$ and $\pi_0$. These laws differ only in the conditional distribution of the chosen $J$. The chain rule for relative entropy therefore gives
\[
\KL(P_\pi\|P_0)
=\E\!\left[\KL\!\left(\pi(\cdot\mid T,A)\,\middle\|\,\pi_0(\cdot\mid A,E)\right)\right]
\leq K.
\]
Under $P_0$, the selected index is independent of the rewards. Thus, for every $t>0$,
\[
\E_{P_0}\!\left[e^{t(R_J-\mu)}\right]
=\Pr(A=\varnothing)+\Pr(A\ne\varnothing)e^{t^2\sigma^2/2}
\leq e^{t^2\sigma^2/2}.
\]
The entropy inequality $\E_P f\leq\KL(P\|Q)+\log\E_Q e^f$, applied with $f=t(R_J-\mu)$, gives
\[
\Delta(\pi)\leq\frac Kt+\frac{t\sigma^2}{2}.
\]
Taking the infimum over $t>0$ gives $\Delta(\pi)\leq\sigma\sqrt{2K}$, including $K=0$ by letting $t\downarrow0$.

\paragraph{The size of the preference dataset.}
Let $V(E)$ be the set of responses appearing in $E$, so $|V(E)|\leq2M$. Conditional on $E$, the labels depend only on rewards in $V(E)$ and, in the noisy model, independent preference randomness. Every response outside $V(E)$ therefore has conditional mean $\mu$ given $(T,A)$. Writing $\pi_i=\pi(i\mid T,A)$ and using the independence of the policy's randomization from $R$ given $(T,A)$, we obtain
\[
\begin{split}
\Delta(\pi)
&=\E\!\left[\sum_{i\in A\cap V(E)}\pi_i(R_i-\mu)\right]\\
&\leq\E Z,\qquad
Z:=\max\bigl(\{0\}\cup\{R_i-\mu:i\in A\cap V(E)\}\bigr).
\end{split}
\]
For $t>0$, independence of $E$, $A$, and the rewards gives
\[
\begin{split}
\E e^{tZ}
&\leq1+\E\!\left[\sum_{i\in A\cap V(E)}e^{t(R_i-\mu)}\right]\\
&=1+p\E|V(E)|e^{t^2\sigma^2/2}
\leq(1+2pM)e^{t^2\sigma^2/2}.
\end{split}
\]
Jensen's inequality and minimization over $t$ now imply
\[
\Delta(\pi)\leq\inf_{t>0}\left\{\frac{\log(1+2pM)}t+\frac{t\sigma^2}{2}\right\}
=\sigma\sqrt{2\log(1+2pM)}.
\]
Combining the two bounds proves the upper bounds in Theorems~\ref{thm:offline_alignment} and~\ref{thm:noisy_alignment}. Notice that neither argument requires the size assumptions used for the lower bounds.

\subsection{Proof of Theorem~\ref{thm:offline_alignment}}
\label{app:exact_proof}

It remains to prove the lower bound. We first relate proxy optimization to the construction, then establish its feasibility and reward gain.

\paragraph{The proxy.}
\label{app:proxy_transfer}
For either preference model, take $\hat r(i;T)=\E[R_i\mid T]$, with $\hat r(0;T)=\mu$. Since $A\perp(R,T)$, every informed policy satisfies
\begin{equation}
\label{eqn:app_proxy_transfer}
\E_\pi R_J
=\E\!\left[\sum_i\pi(i\mid T,A)\E[R_i\mid T,A]\right]
=\E_\pi\hat r(J;T).
\end{equation}
It therefore suffices to construct a feasible informed policy with the claimed gain. An optimizer exists for the following reason. There are only finitely many possible pairs $(T,A)$: the dataset records pair identities and binary labels, and $A\subseteq[n]$. Ignore pairs of probability zero. A policy is then specified by finitely many choice probabilities. A finite divergence budget requires $\pi(i\mid T,A)=0$ whenever $\pi_0(i\mid A,E)=0$. The remaining probabilities are nonnegative and sum to one for each $(T,A)$, so their domain is compact. On this domain, the divergence is continuous: each term $x\log(x/b)$ with $b>0$ extends continuously to zero by assigning it the value zero. Hence the budget constraint defines a closed subset of this compact domain. The feasible set is nonempty because it contains $\pi_0$. Finally, expected proxy reward is a continuous linear function of the choice probabilities, so it must attain its maximum on the feasible set.

\paragraph{The preference dataset and base policy.}
For the remainder of the proof, write
\[
m=\lfloor\sqrt{2M}\rfloor,\qquad
q=p\sqrt{2M},\qquad
h=\min\{K,\log q\}.
\]
We assumed that $m\leq n$ and $q\geq4$. We then choose a fixed set $S$ of $m$ responses and compare every pair in $S$. This requires $\binom m2\leq M$ comparisons. Thus the exact preference dataset reveals the full ranking of $S$.

Let $B=A\cap S$ and $N=|B|\sim\mathrm{Bin}(m,p)$. When $N>0$, take the base policy to be uniform on $B$. When $N=0$, take it to be uniform on $A$, or to select action $0$ if $A=\varnothing$. We will use
\begin{equation}
\label{eqn:app_mean_size}
pm\geq q-1\geq3,\qquad pm\leq q,\qquad
\log(pm)\geq\tfrac12\log q.
\end{equation}
where the last inequality follows from $q-1\geq\sqrt q$ for $q\geq4$.

\paragraph{Selection within the divergence budget.}
For $K\geq\log2$, set $s=\min\{N,\lfloor e^K\rfloor\}$. When $N>0$, we draw an $s$-element subset $U$ uniformly from $B$, independently of the rewards and labels conditional on $A$, and select its highest-ranked element. To bound divergence, consider a reference rule that draws the same subset and then selects uniformly from it. Its marginal distribution on responses is uniform on $B$, exactly the base policy. Conditional on $(T,A,U)$, selecting one member instead of a uniform member has divergence $\log s$. Marginalizing over $U$ cannot increase relative entropy, so
\begin{equation}
\label{eqn:app_subset_kl}
\KL\!\left(\pi(\cdot\mid T,A)\,\middle\|\,\pi_0(\cdot\mid A,E)\right)
\leq\log s\leq K
\qquad(N>0).
\end{equation}
This feasibility holds for any ranking obtained from the preference dataset.

For $0\leq K<\log2$, a subset of size $\lfloor e^K\rfloor$ would give no gain. Instead, when $N\geq2$, sample a pair uniformly from $B$ and choose its higher-ranked element with probability $(1+\theta)/2$, where $\theta=\sqrt K$, choosing the other member otherwise. Follow the base policy when $N<2$. The reference rule again chooses uniformly from the same pair. Its marginal is thus the base policy, and the conditional divergence is
\begin{equation}
\label{eqn:app_pair_kl}
d(\theta):=\frac{1+\theta}{2}\log(1+\theta)
+\frac{1-\theta}{2}\log(1-\theta)
\leq\theta^2=K.
\end{equation}
The inequality follows from $\log x\leq x-1$. Marginalizing over the sampled pair proves feasibility, again for any ranking.

\paragraph{Two estimates for the reward gain.}
For independent standard Gaussians $Z_1,Z_2,\ldots$, define $a_j=\E\max_{1\leq i\leq j}Z_i$ for $j\geq1$, and set $a_0=0$. We use the Gaussian maximal inequalities:
\begin{equation}
\label{eqn:app_gaussian_max}
a_j\geq c_0\sqrt{\log j}\quad(j\geq2),\qquad
a_j\sim\sqrt{2\log j}\quad(j\to\infty),
\end{equation}
for a universal $c_0>0$.

We also need a uniform lower-tail estimate for $N$. Let $\nu=pm\geq3$ and $F=\{N\geq\nu/2\}$. Since $\E N^2\leq\nu^2+\nu$, the inequality
\[
\nu=\E N\leq\nu/2+\sqrt{\E N^2\Pr(F)}
\]
gives
\begin{equation}
\label{eqn:app_binomial_event}
\Pr(F)\geq\frac{\nu}{4(\nu+1)}\geq\frac16,
\qquad N\geq2,\quad \log N\geq\tfrac12\log\nu\quad\text{on }F.
\end{equation}
To see the last two claims, if $3\leq\nu\leq4$, integrality gives $N\geq2\geq\sqrt\nu$ on $F$: if $\nu\geq4$, then $N\geq\nu/2\geq\sqrt\nu$.

\paragraph{The achieved gain.}
Under exact preferences, the subset rule selects its largest gold reward. The subset was drawn independently of rewards, so
\begin{equation}
\label{eqn:app_exact_subset_gain}
\Delta(\pi)=\sigma\E a_s.
\end{equation}
For $K\geq\log2$, the inequality $\lfloor e^K\rfloor\geq e^{K/2}$, together with~\eqref{eqn:app_mean_size} and~\eqref{eqn:app_binomial_event}, implies, on $F$,
\[
\log s\geq\tfrac12\min\{K,\log(pm)\}\geq\tfrac14h.
\]
Since $a_s\geq0$ also outside $F$, we obtain
\[
\Delta(\pi)\geq\sigma\Pr(F)c_0\sqrt{h/4}
\geq\frac{c_0}{12}\sigma\sqrt h.
\]
For $K<\log2$, the sampled pair has independent Gaussian rewards and $\E|R_i-R_j|=2\sigma/\sqrt\pi$. Biasing the uniform choice toward the better response by $\theta=\sqrt K$ therefore gives
\begin{equation}
\label{eqn:app_exact_pair_gain}
\Delta(\pi)=\frac{\sigma\sqrt K}{\sqrt\pi}\Pr(N\geq2)
\geq\frac{\sigma\sqrt K}{6\sqrt\pi}.
\end{equation}
Here $h=K$, since $q\geq4$. Both ranges are thus covered by the universal constant
\begin{equation}
\label{eqn:app_exact_constant}
c_{\mathrm{ex}}:=\min\left\{\frac{c_0}{12},\frac1{6\sqrt\pi}\right\}>0.
\end{equation}
At $K=0$, the pair rule is the base policy and the claimed lower bound is zero. Applying~\eqref{eqn:app_proxy_transfer} completes the proof of Theorem~\ref{thm:offline_alignment}.

\subsection{Proof of Theorem~\ref{thm:noisy_alignment}}
\label{app:noisy_proof}

The upper bound follows from Appendix~\ref{app:common_upper}. For the lower bound, remember the set $S$, base policy, and notation from Appendix~\ref{app:exact_proof}, with $q\geq4$. The subset and pair rules remain feasible for any ordering of $S$. We proceed by bounding the reward lost when this ordering is obtained from a random preference dataset.

\paragraph{Estimating the ranking.}
For $i\in S$, let $w_i$ be its fraction of wins and let
\[
u_i=\E[w_i\mid R]
=\frac1{m-1}\sum_{j\in S\setminus\{i\}}
\ell\!\left(\frac{R_i-R_j}{\beta}\right),
\qquad \ell(x)=\frac1{1+e^{-x}}.
\]
Order responses by $w_i$, breaking ties by index. Set
\[
b_m=2\sigma\sqrt{\log m},\qquad
t_m=\sqrt{\frac{2\log m}{m-1}},
\]
and define
\[
G_R=\left\{\max_{i\in S}|R_i-\mu|\le b_m\right\},
\qquad
G_C=\left\{\max_{i\in S}|w_i-u_i|\le t_m\right\}.
\]
The Gaussian tail bound and a union bound give $\Pr(G_R^c)\le2/m$. Conditional on $R$, each $w_i$ averages $m-1$ independent Bernoulli variables. Hoeffding's inequality and another union bound therefore give
\[
\Pr(G_C^c\mid R)
\le2m\exp\{-2(m-1)t_m^2\}
=\frac2{m^3}.
\]
The scores of different responses need not be independent for this bound.

On $G_R$, every argument of $\ell$ in these scores lies in $[-2b_m/\beta,2b_m/\beta]$. Since $\ell'(x)\ge e^{-|x|}/4$, if $R_i\ge R_j$ then
\begin{align}
(m-1)(u_i-u_j)
&=\ell\!\left(\frac{R_i-R_j}{\beta}\right)
-\ell\!\left(\frac{R_j-R_i}{\beta}\right)\notag\\
&\quad+\sum_{k\in S\setminus\{i,j\}}
\left[\ell\!\left(\frac{R_i-R_k}{\beta}\right)
-\ell\!\left(\frac{R_j-R_k}{\beta}\right)\right]\notag\\
&\ge\frac{m e^{-2b_m/\beta}}{4\beta}(R_i-R_j).
\label{eqn:app_noisy_score_slope}
\end{align}
The mutual comparison contributes two terms to this lower bound, each of the other $m-2$ opponents contributes one. If $w_j\ge w_i$ despite $R_i\ge R_j$, then $u_i-u_j\le2t_m$ on $G_C$. Consequently, on $G_R\cap G_C$, every such inverted pair satisfies
\begin{equation}
R_i-R_j\le\delta_m,
\qquad \delta_m=8\beta e^{2b_m/\beta}t_m.
\label{eqn:app_noisy_inversion}
\end{equation}
This holds simultaneously for every pair. In particular, in any nonempty subset of $S$, the response with the largest score has reward within $\delta_m$ of the largest true reward.

To control the loss off this event, let $D_m=\max_{i\in S}R_i-\min_{i\in S}R_i$ and $U_m=\max_{i\in S}|R_i-\mu|/\sigma$. A Gaussian union bound gives
\[
\E U_m^2
\le\int_0^\infty\min\{1,2me^{-v/2}\}\,dv
=2\log(2m)+2,
\]
so $\E D_m^2\le8\sigma^2(\log(2m)+1)$. Since $\Pr((G_R\cap G_C)^c)\le4/m$, Cauchy--Schwarz shows that the expected loss from maximizing the score rather than the reward in any subset is at most
\begin{equation}
\eta_m
=8\beta e^{(4\sigma/\beta)\sqrt{\log m}}
\sqrt{\frac{2\log m}{m-1}}
+4\sqrt2\sigma\sqrt{\frac{\log(2m)+1}{m}}.
\label{eqn:app_noisy_eta}
\end{equation}
The bound is uniform over the subset, and hence applies to the random subsets used by our policy. It also applies to a random pair. Empty admissible sets within $S$ use the common fallback and incur no loss.

\paragraph{The resulting reward gain.}
For $K\ge\log2$, draw a uniform subset of $B$ of size $s=\min\{N,\lfloor e^K\rfloor\}$ and select its highest-scoring member. The subset is sampled independently of rewards and comparison data conditional on $A$. Equation~\ref{eqn:app_exact_subset_gain} and the uniform loss bound imply
\begin{equation}
\Delta(\pi)\ge\sigma\E a_s-\eta_m,
\label{eqn:app_noisy_subset_gain}
\end{equation}
where $s=0$ on $N=0$ and $a_0=0$. For $0\le K<\log2$, favor the highest-scoring member of a pair sampled uniformly from $B$ with bias $\sqrt K$. This rule is the mixture of selecting that member with probability $\sqrt K$ and selecting uniformly from the pair otherwise. Its loss relative to the exact-ranking rule is therefore at most $\sqrt K\,\eta_m$. By~\eqref{eqn:app_exact_pair_gain},
\begin{equation}
\Delta(\pi)
\ge\sqrt K\left(\frac{\sigma}{\sqrt\pi}\Pr(N\ge2)-\eta_m\right).
\label{eqn:app_noisy_pair_gain}
\end{equation}
The base policy is used when $N\le1$. The factor $\sqrt K$ on the error makes this estimate uniform as $K$ decreases to zero.

\paragraph{A bound for every dataset size.}
Set $r=\beta/\sigma$. Dividing~\eqref{eqn:app_noisy_eta} by $\sigma$ gives
\begin{equation}
\frac{\eta_m}{\sigma}
=8r e^{(4/r)\sqrt{\log m}}\sqrt{\frac{2\log m}{m-1}}
+4\sqrt2\sqrt{\frac{\log(2m)+1}{m}}
\longrightarrow0
\label{eqn:app_noisy_eta_limit}
\end{equation}
for every fixed $r>0$: the exponential factor is $m^{o(1)}$. Recall that $c_{\mathrm{ex}}\le1/(6\sqrt\pi)$ by~\eqref{eqn:app_exact_constant}. Choose an integer $m_0(r)\ge4$ such that, for all $m\ge m_0(r)$,
\[
\frac{\eta_m}{\sigma}
\le\min\left\{\frac{c_{\mathrm{ex}}\sqrt{\log2}}2,
\frac1{12\sqrt\pi}\right\}.
\]
This threshold depends only on $r$. When $K\ge\log2$, we have $h\ge\log2$, and the deterministic subset bound together with~\eqref{eqn:app_noisy_subset_gain} gives $\Delta(\pi)\ge(c_{\mathrm{ex}}/2)\sigma\sqrt h$. When $K<\log2$, we have $h=K$; using $\Pr(N\ge2)\ge1/6$ in~\eqref{eqn:app_noisy_pair_gain} gives the same conclusion.

To cover the remaining finite range $m<m_0(r)$, we sample a pair uniformly from $B$ and use its observed preference, following the base policy if $N<2$. If $K\ge\log2$, select the preferred member, otherwise we favor it with bias $\sqrt K$. The pair calculation in the deterministic proof gives divergence at most $K$ for both rules, regardless of the label. If $D\sim\mathcal N(0,2)$, define
\[
\gamma(r)=\frac12\E\left[D\tanh\!\left(\frac D{2r}\right)\right]>0.
\]
Conditioning on the two rewards shows that following their observed label has expected gain $\sigma\gamma(r)$. Favoring it with bias $\sqrt K$ multiplies this gain by $\sqrt K$. The expectation defining $\gamma(r)$ is finite, and its integrand is positive almost surely. Therefore
\[
\Delta(\pi)\ge\frac{\sigma\gamma(r)}6\min\{\sqrt K,1\}.
\]
In this range $q\le\sqrt{2M}<m+1\le m_0(r)+1$, so
\[
\min\{\sqrt K,1\}
\ge\frac{\sqrt h}{\sqrt{\max\{1,\log(m_0(r)+1)\}}}.
\]
Thus a positive constant valid at every finite budget is
\[
c_{\mathrm{noisy}}(r)
=\min\left\{\frac{c_{\mathrm{ex}}}{2},
\frac{\gamma(r)}{6\sqrt{\max\{1,\log(m_0(r)+1)\}}}\right\}.
\]
All the constructed policies use the same reward-blind base and meet the divergence constraint. Equation~\ref{eqn:app_proxy_transfer} transfers their gain to every constrained proxy optimizer, proving the theorem.

\subsection{Proof of Corollary~\ref{cor:comparison_asymptotics}}
\label{app:comparison_asymptotics}

We compare every pair in $S$ and use the same subset rule as above. Along the sequences in the corollary, $q=p\sqrt{2M}\to\infty$ and $K\to\infty$. Since $q-1\leq pm\leq q$, we have $pm/q\to1$ and $m\to\infty$. Write
\[
h_m:=\min\{K,\log(pm)\}=h+o(1),\qquad h\to\infty.
\]
A binomial Chernoff bound gives $\Pr(N\geq pm/2)\geq1-e^{-pm/8}$. On this event, for all sufficiently large budgets,
\[
s=\min\{N,\lfloor e^K\rfloor\}
\geq\tfrac12\min\{pm,e^K\}=\tfrac12e^{h_m}.
\]
Monotonicity and nonnegativity of $a_j$, followed by~\eqref{eqn:app_gaussian_max}, give
\begin{equation}
\label{eqn:app_asymptotic_gain}
\E a_s\geq(1-e^{-pm/8})a_{\lfloor e^{h_m}/2\rfloor}
\geq(1-o(1))\sqrt{2h}.
\end{equation}
The exact-preference construction has gain $\sigma\E a_s$. For noisy preferences, the construction loses at most $\eta_m$, as established in~\eqref{eqn:app_noisy_subset_gain}. For fixed $\sigma,\beta>0$, $\eta_m=o(1)$, so this loss is negligible relative to $\sigma\sqrt h$. Thus either preference model admits a feasible policy with gain
\[
(1-o(1))\sigma\sqrt{2\min\{K,\log(p\sqrt{2M})\}}.
\]
Thus the posterior-mean proxy transfers this bound to every optimizer by~\eqref{eqn:app_proxy_transfer}.

\section{Additional Experimental Details}
\label{app:experiments}

\subsection{Data construction and topic composition}
\label{app:data}

We use English first-turn prompts from LMArena's \texttt{arena-human-preference-140k}. The final training set contains 1,000 prompts, each contributing 300 disjoint response pairs and hence 600 retained responses. Responses are sampled from Qwen3-4B-Instruct-2507 with temperature~1, top-$p=1$, no top-$k$ restriction, and at most 1,024 new tokens. Pair construction excludes identical responses and repeated text pairs, additional responses are sampled when necessary to obtain 300 eligible pairs.

Table~\ref{tab:data_topics_train_test} gives the prompt-topic distribution.

\begin{table}[t]
    \centering
    \caption{\textbf{Training and evaluation prompt topics.} Counts and percentages within each set, using original Arena annotations. Tags may overlap. The dataset has no finer topic annotations for prompts marked \emph{Topic unspecified}.}
    \label{tab:data_topics_train_test}
    \begin{tabular}{lrr}
        \toprule
        Topic & Training ($n=1,000$) & Evaluation ($n=128$) \\
        \midrule
        Coding & 325 (32.5\%) & 43 (33.6\%) \\
        Mathematics & 81 (8.1\%) & 14 (10.9\%) \\
        Creative writing & 68 (6.8\%) & 8 (6.3\%) \\
        Topic unspecified & 558 (55.8\%) & 67 (52.3\%) \\
        \bottomrule
    \end{tabular}
\end{table}

\subsection{Conditional gold-reward distributions}
\label{app:gold_reward_distributions}

The theoretical model in Section~\ref{sec:it_setup} assumes Gaussian rewards. We examine this idealization using the empirical distribution of gold scores within individual prompts. Each example uses all 8,192 responses in its test-prompt pool, scored by the fixed gold reward model before best-of-$N$ selection.

We choose four prompts independently of their reward values. The examples concern AutoHotkey scrolling, a modal-logic agent, motivational story topics, and railway recruitment exam questions.

For a prompt with scores $r_1,\ldots,r_n$, let $\bar r=n^{-1}\sum_i r_i$ and $s^2=n^{-1}\sum_i(r_i-\bar r)^2$. The normal Q--Q plot compares the ordered standardized scores $z_{(i)}=(r_{(i)}-\bar r)/s$ with
\begin{equation}
    q_i=\Phi^{-1}\!\left(\frac{i-1/2}{n}\right),
    \qquad i=1,\ldots,n,
\end{equation}
where $\Phi$ is the standard normal cumulative distribution function. Gaussian scores would approximately follow the line $z=q$.

Figures~\ref{fig:gold_reward_distributions} and~\ref{fig:gold_reward_qq} show that the conditional gold-reward distributions need not be Gaussian. All four examples have negative empirical skewness, ranging from $-0.75$ to $-0.22$. The creative-writing example has the strongest asymmetry, with skewness $-0.75$ and excess kurtosis $1.19$. Departures from the reference line in the Q--Q plots reveal differences in distributional shape even after matching each prompt's mean and variance. Thus, the empirical scaling results arise in a setting whose reward distributions depart from the Gaussian idealization.

\begin{figure}[t]
    \centering
    \includegraphics[width=\linewidth]{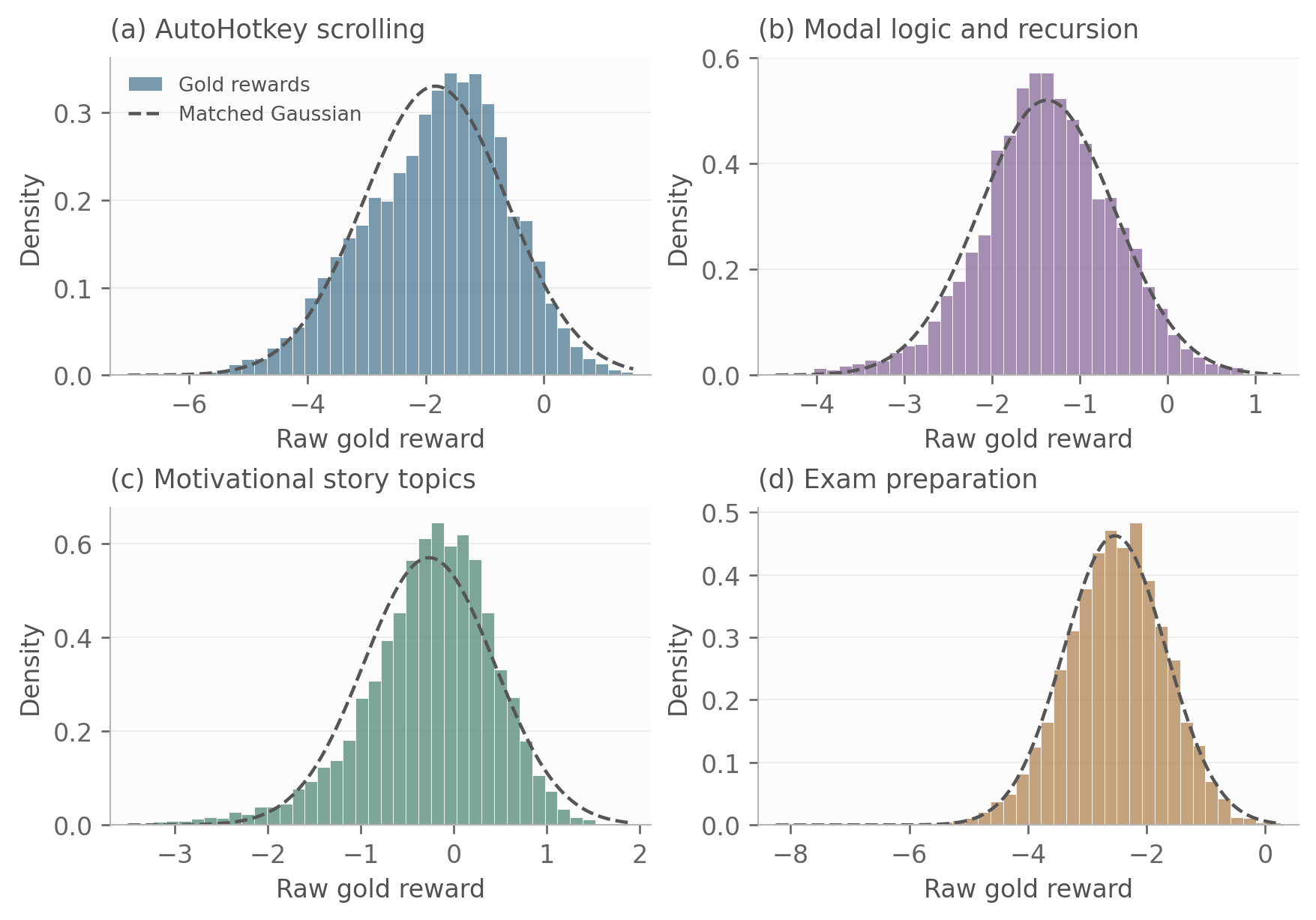}
    \caption{\textbf{Gold-reward distributions for four test prompts.} Each histogram contains all 8,192 raw gold scores for one prompt. The overlaid Gaussian has the same empirical mean and variance. Panels use the same prompt order as Figure~\ref{fig:gold_reward_qq}; prompts are selected by topic and question ID without inspecting rewards.}
    \label{fig:gold_reward_distributions}
\end{figure}

\begin{figure}[t]
    \centering
    \includegraphics[width=\linewidth]{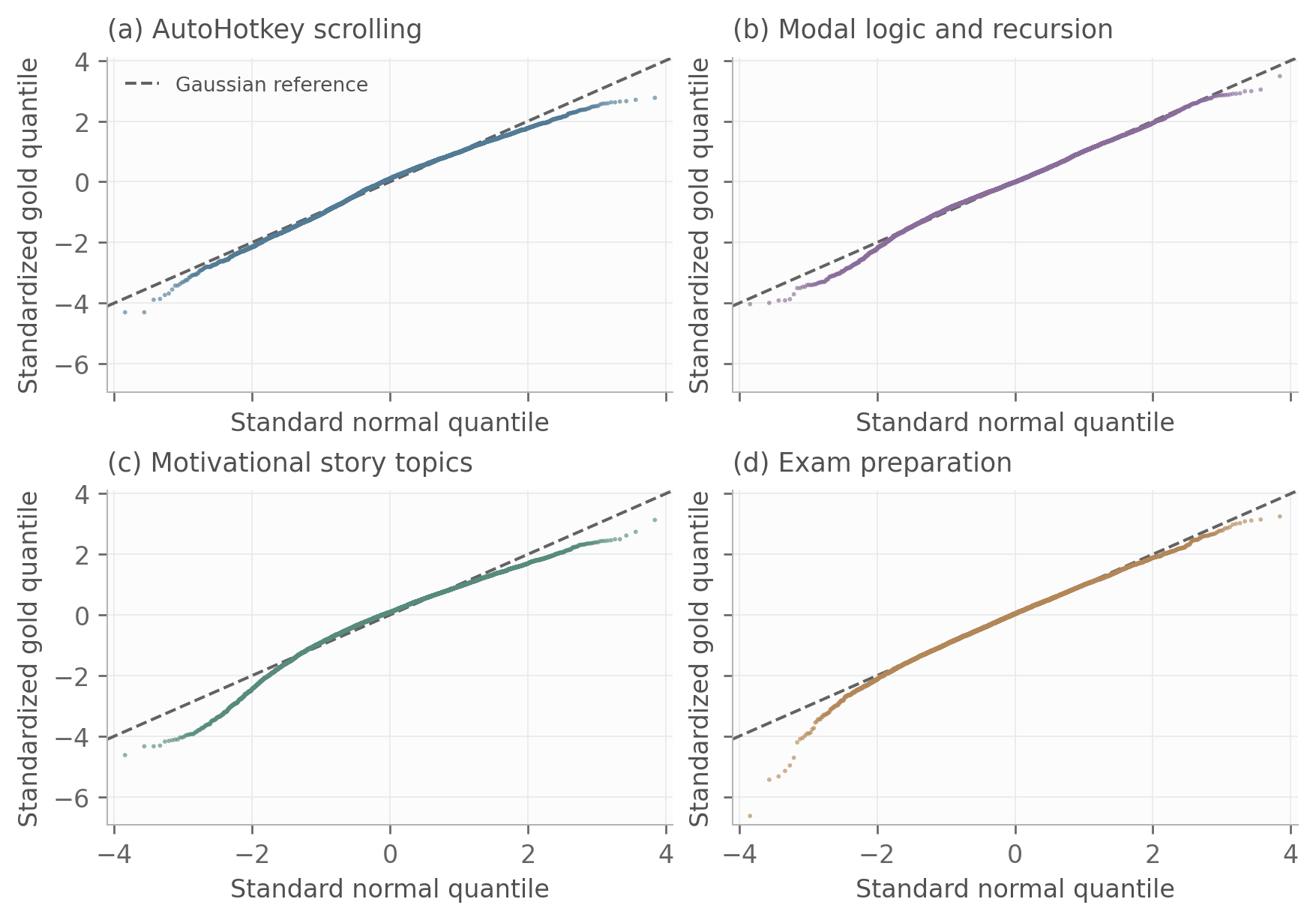}
    \caption{\textbf{Normal Q--Q plots for the same four prompts.} Ordered gold scores, standardized within each prompt, are plotted against standard normal quantiles at probabilities $(i-1/2)/8{,}192$. The diagonal is the Gaussian reference. Systematic departures, particularly in the tails, illustrate that empirical gold rewards need not follow a Gaussian distribution.}
    \label{fig:gold_reward_qq}
\end{figure}

\subsection{Hyperparameters}
\label{app:hyperparameters}

We found reward-model training to work over a broad range of learning rates in preliminary sweeps, from $2\times10^{-6}$ to $4\times10^{-5}$, consistent with the findings of \citet[Appendix~C.2]{ouyang2022training}. We then compared $2\times10^{-6}$ and $10^{-5}$ in 75,000-comparison pilots, selecting by final logistic loss on 512 validation comparisons from 64 held-out prompts distinct from training and test sets. All validations selected $10^{-5}$, which we use for the scaling experiments; \citet{coste2024reward} use the same reward-model learning rate. Table~\ref{tab:hyperparameters} gives the complete shared settings.

\begin{table}[t]
    \centering
    \small
    \setlength{\tabcolsep}{4pt}
    \caption{\textbf{Training hyperparameters.} }
    \label{tab:hyperparameters}
    \begin{tabular}{p{0.36\linewidth}p{0.57\linewidth}}
        \toprule
        Hyperparameter & Value \\
        \midrule
        Trainable parameters & Full backbone and scalar reward head \\
        Objective & Pairwise logistic loss \\
        Optimizer & AdamW \\
        Learning rate / schedule & $10^{-5}$ / constant, no warmup \\
        Adam $(\beta_1,\beta_2)$ / $\epsilon$ & $(0.9,0.999)$ / $10^{-8}$ \\
        Weight decay & $10^{-3}$ on non-normalization matrix weights \\
        Gradient clipping & Global $\ell_2$ norm of 1 \\
        Training epochs & 1 \\
        Effective batch size & 600 comparisons \\
        Model precision & BF16 weights; FP32 gradient accumulation, master weights, and Adam states \\
        Input format & Qwen chat template; thinking disabled \\
        Maximum sequence length & 4,096 tokens; no truncation, right padding \\
        Attention / attention dropout & FlashAttention~2 / 0 \\
        Reward head & Bias-free linear head on the final nonpadding token \\
        Head initialization & $\mathcal{N}(0,0.02^2)$ \\
        \bottomrule
    \end{tabular}
\end{table}

\section{Additional experiments}
\label{app:additional_exp}
Figure~\ref{fig:all_scaling_fits} collects the square-root-logarithmic fits for all ten settings in Table~\ref{tab:scaling_summary}.

Figure~\ref{fig:bon_examples_06b} shows the best-of-$N$ reward curves for four comparison budgets, using 0.6B proxies trained on noiseless and noisy preferences.

\begin{figure}[p]
    \centering
    \includegraphics[width=\linewidth]{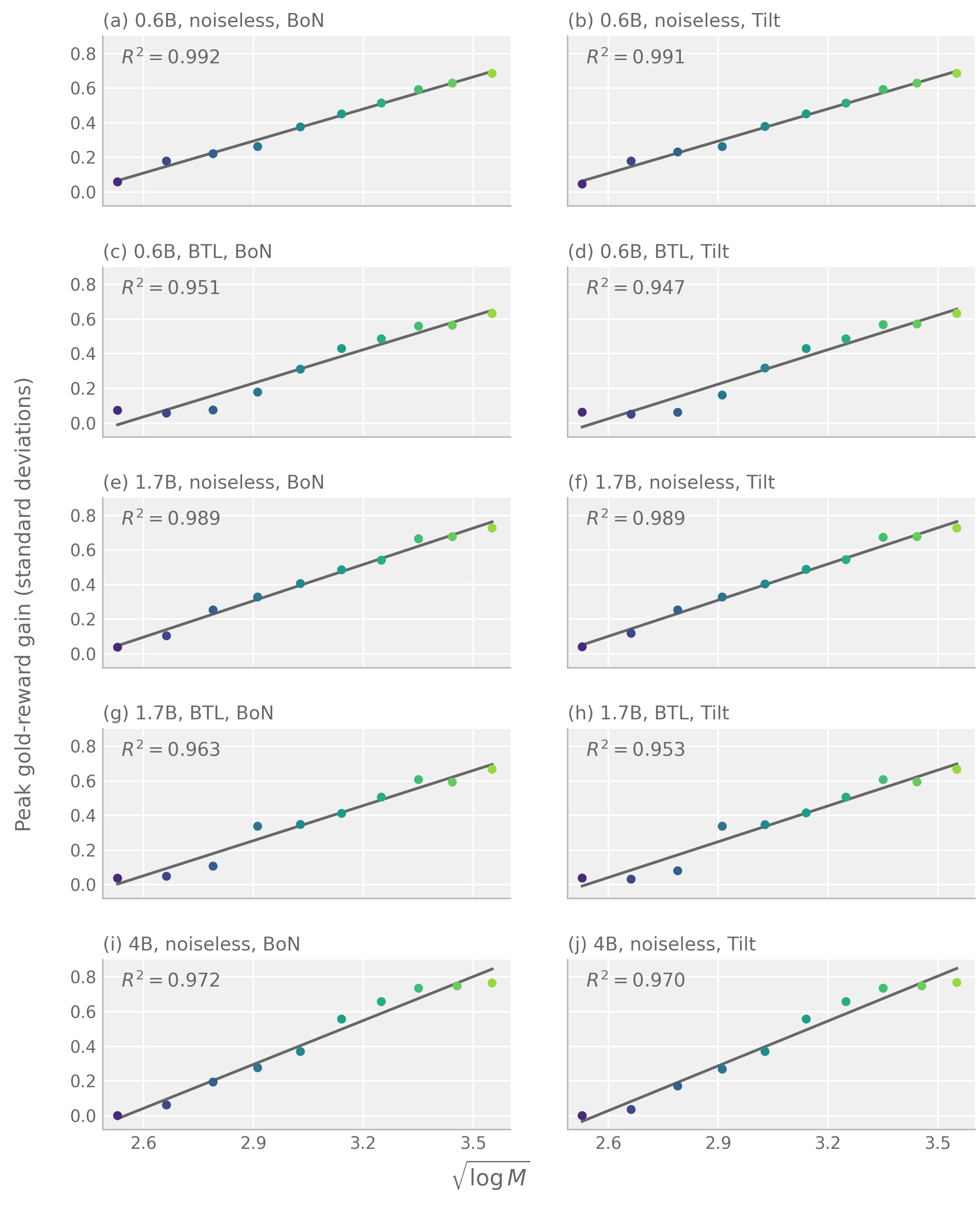}
    \caption{\textbf{Square-root-logarithmic scaling across all settings in Table~\ref{tab:scaling_summary}.} Columns show best-of-$N$ (BoN) and exponential tilting (Tilt); rows vary proxy size and preference noise. Each point is the peak mean gold-reward gain over the evaluated optimization grid. Lines fit $a+b\sqrt{\log M}$ to ten models between $M=600$ and $300{,}000$.}
    \label{fig:all_scaling_fits}
\end{figure}

\begin{figure}[t]
    \centering
    \includegraphics[width=\linewidth]{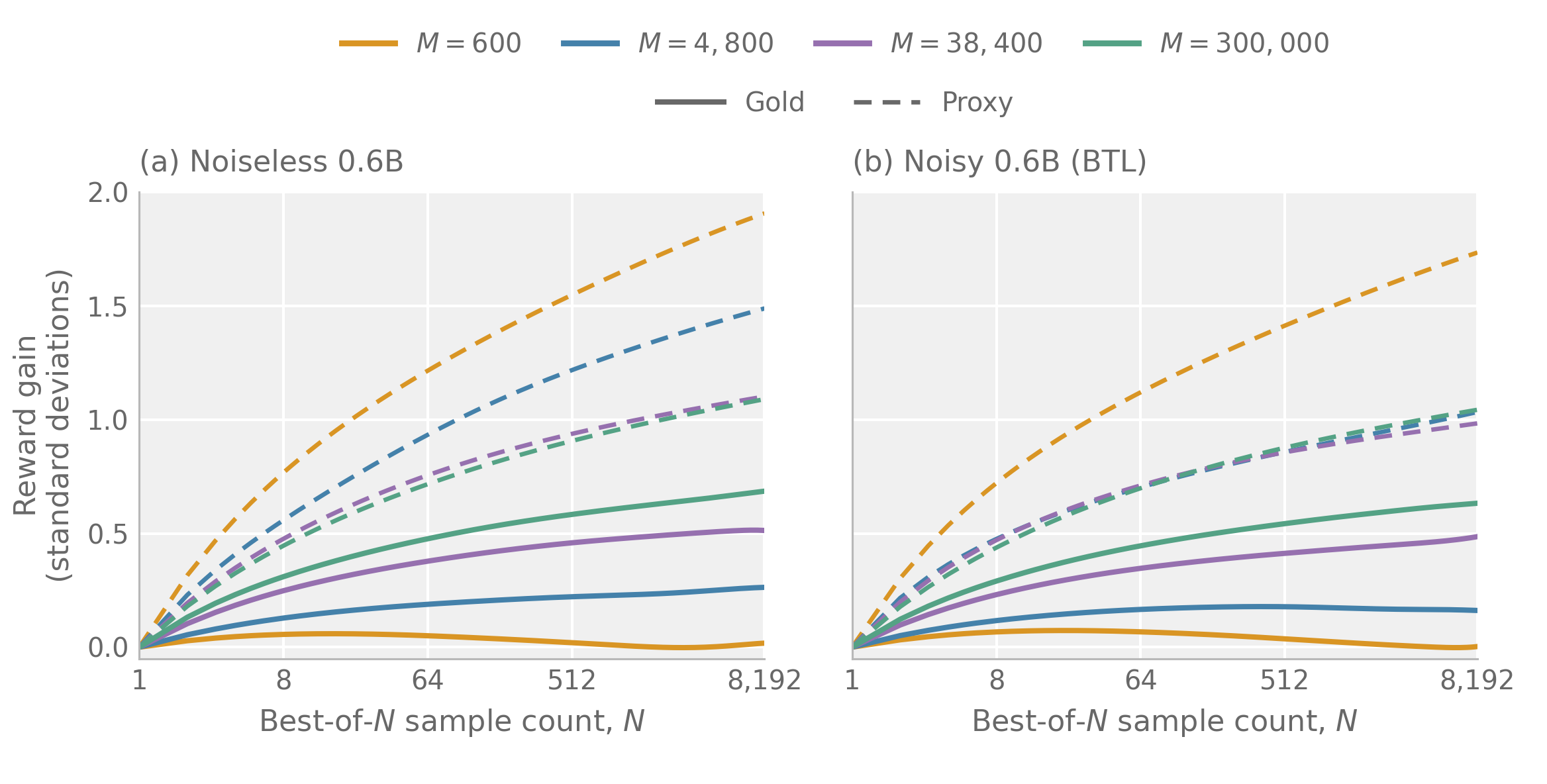}
    \caption{\textbf{Best-of-$N$ optimization with noiseless and noisy preferences.} Each panel shows four 0.6B checkpoints trained on $M=600$, $4{,}800$, $38{,}400$, and $300{,}000$ comparisons. The $N$-axis is logarithmic. Solid and dashed lines show mean gold- and proxy-reward gains over the same evaluation prompts, using the unbiased subset-average estimator on 8,192 responses per prompt. Gains are centered at $N=1$. Gold rewards use the fixed reference standard deviation, and each proxy uses its own standard deviation over the full evaluation pool.}
    \label{fig:bon_examples_06b}
\end{figure}

\end{document}